\documentclass{article}

\PassOptionsToPackage{numbers, compress}{natbib}
\usepackage[final]{nips_2016}

\usepackage[utf8]{inputenc} 
\usepackage[T1]{fontenc}    
\usepackage{hyperref}       
\usepackage{url}            
\usepackage{booktabs}       
\usepackage{amsfonts}       
\usepackage{nicefrac}       
\usepackage{microtype}      

\usepackage{array}
\usepackage{tabulary}
\newcolumntype{K}[1]{>{\centering\arraybackslash}p{#1}}

\usepackage{times}
\usepackage{graphicx} 
\usepackage{subfigure}
\usepackage{amsmath}
\usepackage{amsfonts}
\usepackage{natbib}
\usepackage{amsthm}
\newtheorem{theorem}{Theorem}[section]
\usepackage{dsfont}
\usepackage{multirow}
\usepackage{mathrsfs}

\usepackage{hyperref}

\newtheorem{lemma}[theorem]{Lemma}

\DeclareMathOperator*{\argmin}{argmin}

\usepackage{algorithm}
\usepackage{algorithmic}

\usepackage{xr}
\newcommand{\ex}[1]{\mathop{\mathbb{E}}\left[#1\right]}
\newcommand{\W}{\Lambda}
\newcommand{\X}{\mathcal{X}}
\newcommand{\I}{\mathcal{I}}

\title{Tracking the Best Expert in Non-stationary Stochastic Environments}

\author{
  Chen-Yu Wei\ \ \ \ \ \ \ \ Yi-Te Hong\ \ \ \ \ \ \ \ Chi-Jen Lu \\
  Institute of Information Science\\
  Academia Sinica, Taiwan \\
  \texttt{$\{$bahh723, ted0504, cjlu$\}$@iis.sinica.edu.tw}\\
}

\begin{document}

\maketitle

\begin{abstract}
We study the dynamic regret of multi-armed bandit and experts problem in non-stationary stochastic environments. We introduce a new parameter $\W$, which measures the total statistical variance of the loss distributions over $T$ rounds of the process, and study how this amount affects the regret. We investigate the interaction between $\W$ and $\Gamma$, which counts the number of times the distributions change, as well as $\W$ and $V$, which measures how far the distributions deviates over time. One striking result we find is that even when $\Gamma$, $V$, and $\Lambda$ are all restricted to constant, the regret lower bound in the bandit setting still grows with $T$. The other highlight is that in the full-information setting, a constant regret becomes achievable with constant $\Gamma$ and $\Lambda$, as it can be made independent of $T$, while with constant $V$ and $\Lambda$, the regret still has a $T^{1/3}$ dependency. We not only propose algorithms with upper bound guarantee, but prove their matching lower bounds as well.
\end{abstract}

\section{Introduction}

Many situations in our daily life require us to make repeated decisions which result in some losses corresponding to our chosen actions. This can be abstracted as the well-known online decision problem in machine learning \cite{CBL06}. Depending on how the loss vectors are generated, two different worlds are usually considered. In the adversarial world, loss vectors are assumed to be deterministic and controlled by an adversary, while in the stochastic world, loss vectors are assumed to be sampled independently from some distributions. In both worlds, good online algorithms are known which can achieve a regret of about $\sqrt{T}$ over $T$ time steps, where the regret is the difference between the total loss of the online algorithm and that of the best offline one. Another distinction is about the information the online algorithm can receive after each action. In the full-information setting, it gets to know the whole loss vector of that step, while in the bandit setting, only the loss value of the chosen action is received. Again, in both settings, a regret of about $\sqrt{T}$ turns out to be achievable.

While the regret bounds remain in the same order in those general scenarios discussed above, things become different when some natural conditions are considered. One well-known example is that in the stochastic multi-armed bandit (MAB) problem, when the best arm (or action) is substantially better than the second best, with a constant gap between their means, then a much lower regret, of the order of $\log T$, becomes possible. This motivates us to consider other possible conditions which can have finer characterization of the problem in terms of the achievable regret.

In the stochastic world, most previous works focused on the stationary setting, in which the loss (or reward) vectors are assumed to be sampled from the same distribution for all time steps. With this assumption, although one needs to balance between exploration and exploitation in the beginning, after some trials, one can be confident about which action is the best and rest assured that there are no more surprises. On the other hand, the world around us may not be stationary, in which existing learning algorithms for the stationary case may no longer work. In fact, in a non-stationary world, the dilemma between exploration and exploitation persists as the underlying distribution may drift as time evolves. How does the non-stationarity affect the achievable regret? How does one measure the degree of non-stationarity?

In this paper, we answer the above questions through the notion of dynamic regret, which measures the algorithm's performance against an offline algorithm allowed to select the best arm at every step.

\paragraph{Related Works.}

One way to measure the non-stationarity of a sequence of distributions is to count the number of times the distribution at a time step differs from its previous one. Let $\Gamma-1$ be this number so that the whole time horizon can be partitioned into $\Gamma$ intervals, with each interval having a stationary distribution. In the bandit setting, a regret of about $\sqrt{\Gamma T}$ is achieved by the EXP3.S algorithm in \cite{auer2002nonstochastic}, as well as the discounted UCB and sliding-window UCB algorithms in \cite{garivier2011upper}. The dependency on $T$ can be refined in the full-information setting: AdaNormalHedge \cite{luo2015achieving}, Adapt-ML-Prod \cite{Gaillard2014}, and Squint \cite{koolen2015second} can all achieve regret in the form of $\sqrt{\Gamma C}$, where $C$ is the total first-order \cite{luo2015achieving} or second-order excess loss \cite{Gaillard2014, koolen2015second}, which is upper-bounded by $T$. From a slightly different Online Mirror Descent approach, \cite{jadbabaie2015online} can also achieve a regret of about $\sqrt{\Gamma D}$, where $D$ is the sum of differences between consecutive loss vectors.

Another measure of non-stationarity, denoted by $V$, is to compute the difference between the means of consecutive distributions and sum them up. Note that this allows the possibility for the best arm to change frequently, with a very large $\Gamma$, while still having similar distributions with a small $V$. For such a measure $V$, \cite{gur2014stochastic} provided a bandit algorithm which achieves a regret of about $V^{1/3} T^{2/3}$. This regret upper bound is unimprovable in general even in the full-information setting, as a matching lower bound was shown in \cite{besbes2015non}. Again, \cite{jadbabaie2015online} refined the upper bound in the full-information setting through the introduction of $D$, achieving the regret of about $\sqrt[3]{\tilde{V}DT}$, for a parameter $\tilde{V}$ different but related to $V$: $\tilde{V}$ calculates the sum of differences between consecutive \textit{realized} loss vectors, while $V$ measures that between \textit{mean} loss vectors. This makes the results of \cite{gur2014stochastic} and \cite{jadbabaie2015online} incomparable. The problem stems from the fact that \cite{jadbabaie2015online} considers the traditional adversarial setting, while \cite{gur2014stochastic} studies the non-stationary stochastic setting. In this paper, we will provide a framework that bridges these two seemingly disparate worlds.

\paragraph{Our Results.}

We base ourselves in the stochastic world with non-stationary distributions, characterized by the parameters $\Gamma$ and $V$. In addition, we introduce a new parameter $\W$, which measures the total statistical variance of the distributions. Note that traditional adversarial setting corresponds to the case with $\W=0$ and $\Gamma \approx V \approx T$, while the traditional stochastic setting has $\W \approx T$ and $\Gamma = V = 1$. Clearly, with a smaller $\W$, the learning problem becomes easier, and we would like to understand the tradeoff between $\W$ and other parameters, including $\Gamma$, $V$, and $T$. In particular, we would like to know how the bounds described in the related works would be changed. Would all the dependency on $T$ be replaced by $\W$, or would only some partial dependency on $T$ be shifted to $\W$?

First, we consider the effect of the variance $\W$ with respect to the parameter $\Gamma$. We show that in the full-information setting, a regret of about $\sqrt{\Gamma \W}+\Gamma$ can be achieved, which is independent of $T$. On the other hand, we show a sharp contrast that in the bandit setting, the dependency on $T$ is unavoidable, and a lower bound of the order of $\sqrt{\Gamma T}$ exists. That is, even when there is no variance in distributions, with $\W=0$, and the distributions only change once, with $\Gamma=2$, any bandit algorithm cannot avoid a regret of about $\sqrt{T}$, while a full-information algorithm can achieve a constant regret independent of $T$.

Next, we study the tradeoff between $\W$ and $V$. We show that in the bandit setting, a regret of about $\sqrt[3]{\W V T} + \sqrt{VT}$ is achievable. Note that this recovers the $V^{1/3} T^{2/3}$ regret bound of \cite{gur2014stochastic} as $\W$ is at most of the order of $T$, but our bound becomes better when $\W$ is much smaller than $T$. Again, one may notice the dependency on $T$ and wonder if this can also be removed in the full-information setting. We show that in the full-information setting, the regret upper bound and lower bound are both about $\sqrt[3]{\Lambda VT}+V$. Our upper bound is incomparable to the $\sqrt[3]{\tilde{V}DT}$ bound of \cite{jadbabaie2015online}, since their adversarial setting corresponds to $\Lambda=0$ and their $D$ can be as large as $T$ in our setting. Moreover, we see that while the full-information regret bound is slightly better than that in the bandit setting, there is still an unavoidable $T^{1/3}$ dependency.

Our results provide a big picture of the regret landscape in terms of the parameters $\W, \Gamma, V$, and $T$, in both full-information and bandit settings. A table summarizing our bounds as well as previous ones is given in Appendix~\ref{app:list} in the supplementary material. Finally, let us remark that our effort mostly focuses on characterizing the achievable (minimax) regrets, and most of our upper bounds are achieved by algorithms which need the knowledge of the related parameters and may not be practical. To complement this, we also propose a parameter-free algorithm, which still achieve a good regret bound and may have independent interest of its own.

\section{Preliminaries}
\label{sec_pre}


Let us first introduce some notations. For an integer $K>0$, let $[K]$ denote the set $\{1,\dots,K\}$. For a vector $\ell \in \mathbb{R}^K$, let $\ell_i$ denote its $i$'th component. When we need to refer to a time-indexed vector $\ell_t \in \mathbb{R}^K$, we will write $\ell_{t,i}$ to denote its $i$'th component. We will use the indicator function $\mathds{1}_\mathcal{C}$ for a condition $\mathcal{C}$, which gives the value $1$ if $\mathcal{C}$ holds and $0$ otherwise. For a vector $\ell$, we let $\|\ell\|_b$ denote its $L_b$-norm. While standard notation $\mathcal{O}(\cdot)$ is used to hide constant factors, we will use the notation $\mathcal{\tilde{O}}(\cdot)$ to hide logarithmic factors.

Next, let us describe the problem we study in this paper. Imagine that a learner is given the choice of a total of $K$ actions, and has to play iteratively for a total of $T$ steps. At step $t$, the learner needs to choose an action $a_t\in [K]$, and then suffers a corresponding loss $\ell_{t,i} \in [0,1]$, which is independently drawn from a non-stationary distribution with expected loss $\mathbb{E}[\ell_{t,i}]=\mu_{t,i}$, which may drift over time. After that, the learner receives some feedback from the environment. In the full-information setting, the feedback gives the whole loss vector $\ell_t=(\ell_{t,1},...,\ell_{t,K})$, while in the bandit setting, only the loss $\ell_{t,a_t}$ of the chosen action is revealed. A standard way to evaluate the learner's performance is to measure her (or his) regret, which is the difference between the total loss she suffers and that of an offline algorithm. While most prior works consider offline algorithms which can only play a fixed action for all the steps, we consider stronger offline algorithms which can take different actions in different steps. Our consideration is natural for non-stationary distributions, although this would make the regret large when compared to such stronger offline algorithms. Formally, we measure the learner's performance by its expected \textit{dynamic pseudo-regret}, defined as
$\sum_{t=1}^{T} \ex{\ell_{t,a_t}-\ell_{t,u_t^*}} = \sum_{t=1}^T \left(\mu_{t,a_t}-  \mu_{t,u_t^*}\right),$
where $u_t^* = \arg\min_i{\mu_{t,i}}$ is the best action at step $t$. For convenience, we will simply refer it as the regret of the learner later in our paper.

We will consider the following parameters characterizing different aspects of the environments:
\begin{equation}\label{eq:var}
\Gamma = 1+\sum_{t=2}^T\mathds{1}_{\mu_t\neq \mu_{t-1}}, V =\sum_{t=1}^T \|\mu_{t}-\mu_{t-1}\|_\infty, \mbox{ and } \W =\sum_{t=1}^T \mathbb{E}\left[\|\ell_{t}-\mu_{t}\|_2^2\right],
\end{equation}
where we let $\mu_0$ be the all-zero vector. Here, $\Gamma-1$ is the number of times the distributions switch, $V$ measures the distance the distributions deviate, and $\W$ is the total statistical variance of these $T$ distributions.
We will call distributions with a small $\Gamma$ switching distributions,
while we will call distributions with a small $V$ drifting distributions and call $V$ the total drift of the distributions.

Finally, we will need the following large deviation bound, known as empirical Bernstein inequality.

\begin{theorem} \cite{DBLP:conf/colt/MaurerP09}
\label{thm:emp_bernstein}
Let $X=(X_1,...,X_n)$ be a vector of independent random variables taking values in $[0,1]$, and let $\W_X = \sum_{1\le i<j\le n} (X_i -X_j)^2/(n(n-1))$. Then for any $\delta>0$, we have
$$\Pr\left[ \left|\sum_{i=1}^n \frac{\ex{X_i} -  X_i}{n} \right| > \rho(n,\W_X,\delta) \right] \le \delta, \;\mbox{ for } \rho(n, \W, \delta) = \sqrt{\frac{2 \W \log\frac{2}{\delta}}{n}}+\frac{7\log\frac{2}{\delta}}{3(n-1)}.$$
\end{theorem}

\section{Algorithms}

We would like to characterize the achievable regret bounds for both switching and drifting distributions, in both full-information and bandit settings. In particular, we would like to understand the interplay among the parameters $\Gamma, V, \W$, and $T$, defined in (\ref{eq:var}). The only known upper bound which is good enough for our purpose is that by \cite{garivier2011upper} for switching distributions in the bandit setting, which is close to the lower bound in our Theorem~\ref{thm:bandit-low}. In subsection~\ref{sec:drift1}, we provide a bandit algorithm for drifting distributions which achieves an almost optimal regret upper bound, when given the parameters $V,\W,T$. In subsection~\ref{sec:full}, we provide a full-information algorithm which works for both switching and drifting distributions. The regret bounds it achieves are also close to optimal, but it again needs the knowledge of the related parameters. To complement this, we provide a full-information algorithm in subsection~\ref{subsec:parameter_free}, which does not need to know the parameters but achieves slightly larger regret bounds.


\subsection{Parameter-Dependent Bandit Algorithm} \label{sec:drift1}

In this subsection, we consider drifting distributions parameterized by $V$ and $\W$. Our main result is a bandit algorithm which achieves a regret of about $\sqrt[3]{\W VT} + \sqrt{VT}$.
As we aim to achieve smaller regrets for distributions with smaller statistical variances, we adopt a variant of the UCB algorithm developed by \cite{audibert2009exploration}, called UCB-V, which takes variances into account when building its confidence interval.

Our algorithm divides the time steps into $T/B$ intervals $\I_1, \dots, \I_{T/B}$, each having $B$ steps,\footnote{For simplicity of presentation, let us assume here and later in the paper that taking divisions and roots to produce blocks of time steps all yield integers. It is easy to modify our analysis to the general case without affecting the order of our regret bound.} with
\begin{equation}\label{eq:B}
B = \sqrt[3]{K^2 \Lambda T/V^2} \mbox{ if } K \Lambda^2 \ge TV \mbox{ and } B=\sqrt{KT/V} \mbox{ otherwise.}
\end{equation}
For each interval, our algorithm clears all the information from previous intervals, and starts a fresh run of UCB-V. More precisely, before step $t$ in an interval $\I$, it maintains for each arm $i$ its empirical mean $\hat{\mu}_{t,i}$, empirical variance $\hat{\W}_{t,i}$, and size of confidence interval $\lambda_{t,i}$, defined as
\begin{equation}\label{eq:mean}
\hat{\mu}_{t,i} = \sum_{s \in S_{t,i}} \frac{\ell_{s,i}}{|S_{t,i}|},\; \hat{\W}_{t,i} = \sum_{r,s \in S_{t,i}} \frac{(\ell_{r,i}-\ell_{s,i})^2}{|S_{t,i}|(|S_{t,i}|-1)},\; \mbox{ and } \lambda_{t,i} = \rho(|S_{t,i}|, \hat{\W}_{t,i}, \delta),
\end{equation}
where $S_{t,i}$ denotes the set of steps before $t$ in $\I$ that arm $i$ was played, and $\rho$ is the function given in Theorem~\ref{thm:emp_bernstein}. Here we use the convention that $\hat{\mu}_{t,i} = 0$ if $|S_{t,i}|=0$, while $\hat{\W}_{t,i} = 0$ and $\lambda_{t,i} = 1$ if $|S_{t,i}| \le 1$.
Then at step $t$, our algorithm selects the optimistic arm
$$a_t:=\argmin_i (\hat{\mu}_{t,i}-\lambda_{t,i}),$$
receives the corresponding loss, and updates the statistics.

\begin{algorithm}[t]
  \caption{Rerun-UCB-V}
  \label{alg:bandit}
\begin{algorithmic}
   \STATE {\bfseries Initialization:} Set $B$ according to (\ref{eq:B}) and $\delta=1/(KT)$.

   \FOR{$m=1, \dots, T/B$}
   \FOR{$t=(m-1)B+1, \dots, mB$}
   \STATE Choose arm $a_t:=\argmin_i (\hat{\mu}_{t,i}-\lambda_{t,i})$, with $\hat{\mu}_{t,i}$ and $\lambda_{t,i}$ computed according to (\ref{eq:mean}).
   \ENDFOR
   \ENDFOR
\end{algorithmic}
\end{algorithm}

Our algorithm is summarized in Algorithm~\ref{alg:bandit}, and its regret is guaranteed by the following, which we prove in Appendix~\ref{app:bandit-u} in the supplementary material.

\begin{theorem} \label{thm:bandit-u}
The expected regret of Algorithm~\ref{alg:bandit} is at most
$\mathcal{\tilde{O}}(\sqrt[3]{K^2 \W VT}+\sqrt{KVT}).$
\end{theorem}

\subsection{Parameter-Dependent Full-Information Algorithms} \label{sec:full}

In this subsection, we provide full-information algorithms for switching and drifting distributions. In fact, they are based on an existing algorithm from \cite{ChiangYLMLJZ12},
which is known to work in a different setting: the loss vectors are deterministic and adversarial, and the offline comparator cannot switch arms. In that setting, one of their algorithms, based on gradient-descent (GD), can achieve a regret of $\mathcal{O}(\sqrt{D})$ where $D=\sum_t \|\ell_t - \ell_{t-1}\|_2^2$, which is small when the loss vectors have small deviation. Our first observation is that their algorithm in fact can work against a dynamic offline comparator which switches arms less than $N$ times, given any $N$, with its regret becoming $\mathcal{O}(\sqrt{N D})$. Our second observation is that when $\W$ is small, each observed loss vector $\ell_t$ is likely to be close to its true mean $\mu_t$, and when $V$ is small, $\ell_t$ is likely to be close to $\ell_{t-1}$. These two observations make possible for us to adopt their algorithm to our setting.

We show the first algorithm in Algorithm~\ref{alg:dev},
with the feasible set $\X$ being the probability simplex. The idea is to use $\ell_{t-1}$ as an estimate for $\ell_t$ to move $\hat{x}_t$ further in a possibly beneficial direction. Its regret is guaranteed by the following, which we prove in Appendix~\ref{app:dev} in the supplementary material.

\begin{algorithm}[t]
\caption{Full-information GD-based algorithm} \label{alg:dev}
\begin{algorithmic}
\STATE {\bf Initialization:} Let $x_1 = \hat{x}_1 = (1/K, \dots, 1/K)^\top$.
\FOR{$t=1,2,\dots,T$}
    \STATE Play $\hat{x}_t = \arg\min_{\hat{x}\in\X} (\langle \ell_{t-1}, \hat{x}\rangle + \frac{1}{\eta_{t}} \|\hat{x} -x_{t}\|_2^2)$, and then receive loss vector $\ell_t$.
    \STATE Update $x_{t+1} = \arg\min_{x\in\X} (\langle \ell_t, x\rangle + \frac{1}{\eta_t} \|x-x_t\|_2^2).$
\ENDFOR
\end{algorithmic}
\end{algorithm}

\begin{theorem} \label{thm:dev}
For switching distributions parameterized by $\Gamma$ and $\W$, the regret of Algorithm~\ref{alg:dev} with $\eta_t=\eta= \sqrt{\Gamma/(\W+K\Gamma)}$, is at most $\mathcal{O}(\sqrt{\Gamma \W} + \sqrt{K}\Gamma)$.
\end{theorem}

Note that for switching distributions, the regret of Algorithm~\ref{alg:dev} does not depend on $T$, which means that it can achieve a constant regret for constant $\Gamma$ and $\W$.
Let us remark that although using a variant based on multiplicative updates could result in a better dependency on $K$, an additional factor of $\log T$ would then emerge when using existing techniques for dealing with dynamic comparators.

For drifting distributions, one can show that Algorithm~\ref{alg:dev} still works and has a good regret bound. However, a slightly better bound can be achieved as we describe next. The idea is to divide the time steps into $T/B$ intervals of size $B$, with $B=\sqrt[3]{\W T/V^2}$ if $\W T> V^2$ and $B=1$ otherwise, and re-run Algorithm~\ref{alg:dev} in each interval with an adaptive learning rate. One way to have an adaptive learning rate can be found in \cite{jadbabaie2015online}, which works well when there is only one interval. A natural way to adopt it here is to reset the learning rate at the start of each interval, but this does not lead to a good enough regret bound as it results in some constant regret at the start of every interval. To avoid this, some careful changes are needed. Specifically, in an interval $[t_1, t_2]$, we run Algorithm~\ref{alg:dev} with the learning rate reset as 
$$\eta_t = 1/ \sqrt{4\sum_{\tau=t_1}^{t-1} \lVert \ell_\tau -\ell_{\tau-1}\rVert ^2_2}$$
for $t > t_1$, with $\eta_{t_1} = \infty$ initially for every interval. This has the benefit of having small or even no regret at the start of an interval when the loss vectors across the boundary have small or no deviation. The regret of this new algorithm is guaranteed by the following, which we prove in Appendix~\ref{app:free} in the supplementary material.


\begin{theorem}\label{thm:full_info_given_V}
For drifting distributions parameterized by $V$ and $\W$, the regret of
this new algorithm is at most $\mathcal{O}(\sqrt[3]{V\W T} + \sqrt{K}V)$.
\end{theorem}

\subsection{Parameter-Free Full-Information Algorithm} \label{subsec:parameter_free}

The reason that our algorithm for Theorem~\ref{thm:full_info_given_V} needs the related parameters is to set its learning rate properly. To have a parameter-free algorithm, we would like to adjust the learning rate dynamically in a data-driven way. One way for doing this can be found in \cite{Gaillard2014}, which is based on the multiplicative updates variant of the mirror-descent algorithm. It achieves a static regret of about $\sqrt{\sum_t r_{t,k}^2}$ against any expert $k$, where $r_{t,k} = \langle p_t, \ell_t\rangle - \ell_{t,k}$ is its instantaneous regret for playing $p_t$ at step $t$.
However, in order to work in our setting, we would like the regret bound to depend on $\ell_{t}-\ell_{t-1}$ as seen previously. This suggests us to modify the Adapt-ML-Prod algorithm of \cite{Gaillard2014} using the idea of \cite{ChiangYLMLJZ12}, which takes $\ell_{t-1}$ as an estimate of $\ell_{t}$ to move $p_t$ further in an optimistic direction.

Recall that the algorithm of \cite{Gaillard2014} maintains a separate learning rate $\eta_{t,k}$ for each arm $k$ at time $t$, and it updates the weight $w_{t,k}$ as well as $\eta_{t,k}$ using the instantaneous regret $r_{t,k}$. To modify the algorithm using the idea of \cite{ChiangYLMLJZ12}, we would like to have an estimate $m_{t,k}$ for $r_{t,k}$ in order to move $p_{t,k}$ further using $m_{t,k}$ and update the learning rate accordingly. More precisely, at step $t$, we now play $p_t$, with
\begin{eqnarray}  \label{eq:p_t}
p_{t,k}= \eta_{t-1,k}\tilde{w}_{t-1,k} / \langle \eta_{t-1}, \tilde{w}_{t-1} \rangle
\;\mbox{ where } \tilde{w}_{t-1,k} = w_{t-1,k} \exp(\eta_{t-1,k}m_{t,k}),
\end{eqnarray}
which uses the estimate $m_{t,k}$ to move further from $w_{t-1,k}$.
Then after receiving the loss vector $\ell_t$, we update each weight
\begin{equation}\label{eq:w_t}
w_{t,k}=\left(w_{t-1,k}\exp\left(\eta_{t-1,k}r_{t,k}- \eta_{t-1,k}^2(r_{t,k}-m_{t,k})^2\right)\right)^{\eta_{t,k}/\eta_{t-1,k}}
\end{equation}
as well as each learning rate
\begin{equation}\label{eq:eta_t}
    \eta_{t,k}=\min\left\{1/4,\sqrt{(\ln K)/\left(1+{\sum}_{s \in [t]}(r_{s,k}-m_{s,k})^2\right)}\right\}.
\end{equation}
Our algorithm is summarized in Algorithm~\ref{alg:Op_Prod}, and we will show that it achieves a regret of about $\sqrt{\sum_t (r_{t,k}-m_{t,k})^2}$ against arm $k$.
It remains to choose an appropriate estimate $m_{t,k}$. One attempt is to have $m_{t,k} = r_{t-1,k}$, but $r_{t,k}- r_{t-1,k} = (\langle p_t, \ell_t \rangle - \ell_{t,k}) -  (\langle p_{t-1}, \ell_{t-1} \rangle - \ell_{t-1,k})$, which does not lead to a desirable bound.
The other possibility is to set $m_{t,k} = \langle p_t, \ell_{t-1} \rangle - \ell_{t-1,k}$, which can be shown to have
$(r_{t,k} - m_{t,k})^2 \le (2\|\ell_t - \ell_{t-1}\|_\infty)^2$. However, it is not clear how to compute such $m_{t,k}$ because it depends on $p_{t,k}$ which in turns depends on $m_{t,k}$ itself. Fortunately, we can approximate it efficiently in the following way.

Note that the key quantity is $\langle p_t, \ell_{t-1}\rangle$. Given its value $\alpha$, $\tilde{w}_{t-1,k}$ and $p_{t,k}$ can be seen as functions of $\alpha$, defined according to \eqref{eq:p_t} as $\tilde{w}_{t-1,k}(\alpha) =w_{t-1,k}\exp(\eta_{t-1,k}(\alpha-\ell_{t-1,k}))$ and $p_{t,k}(\alpha) = \eta_{t-1,k}\tilde{w}_{t-1,k}(\alpha) / \sum_i \eta_{t-1,i}\tilde{w}_{t-1,i}(\alpha)$. Then we would like to show the existence of $\alpha$ such that $\langle p_t(\alpha), \ell_{t-1}\rangle = \alpha$ and to find it efficiently. For this, consider the function $f(\alpha) = \langle p_t(\alpha), \ell_{t-1}\rangle$, with $p_t(\alpha)$ defined above. It is easy to check that $f$ is a continuous function bounded in $[0,1]$, which implies the existence of some fixed point $\alpha \in [0,1]$ with $f(\alpha)=\alpha$. Using a binary search, such an $\alpha$ can be approximated within error $1/T$ in $\log T$ iterations. As such a small error does not affect the order of the regret, we will ignore it for simplicity of presentation, and assume that we indeed have $\langle p_t, \ell_{t-1}\rangle$ and hence $m_{t,k} = \langle p_t, \ell_{t-1} \rangle - \ell_{t-1,k}$ without error.

\begin{algorithm}[t]
\caption{Optimistic-Adapt-ML-Prod} \label{alg:Op_Prod}
\begin{algorithmic}
\STATE {\bf Initialization:} Let $w_{0,k} = 1/K$ and $\ell_{0,k}=0$ for every $k\in[K]$.
\FOR{$t=1,2,\dots,T$}
    \STATE Play $p_t$ according to (\ref{eq:p_t}), and then receive loss vector $\ell_t$.
    \STATE Update each weight $w_{t,k}$ according to (\ref{eq:w_t}) and each learning rate $\eta_{t,k}$ according to (\ref{eq:eta_t}).
\ENDFOR
\end{algorithmic}
\end{algorithm}


Then we have the following regret bound (c.f. \cite[Corollary 4]{Gaillard2014}), which we prove in Appendix~\ref{app:RTK} in the supplementary material.
\begin{theorem}
\label{thm_RTk_bound}
The static regret of Algorithm \ref{alg:Op_Prod} w.r.t. any arm (or expert) $k\in [K]$ is at most
$$\hat{\mathcal{O}}\left(\sqrt{{\sum}_{t\in [T]} (r_{t,k}-m_{t,k})^2 \ln K}+\ln K\right) \le \hat{\mathcal{O}}\left(\sqrt{{\sum}_{t\in [T]} \| \ell_t - \ell_{t-1}\|_\infty^2 \ln K}+\ln K\right),$$
where
the notation $\hat{\mathcal{O}}(\cdot)$ hides a $\ln \ln T$ factor.
\end{theorem}

The regret in the theorem above is measured against a fixed arm. To achieve a dynamic regret against an offline algorithm which can switch arms, one can use a generic reduction to the so-called sleeping experts problem.
In particular, we can use the idea in \cite{Gaillard2014} by creating $\tilde{K}=KT$ sleeping experts, and run our Algorithm~\ref{alg:Op_Prod} on these $\tilde{K}$ experts (instead of on the $K$ arms). More precisely, each sleeping expert is indexed by some pair $(s,k)$, and it is asleep for steps before $s$ and becomes awake for steps $t\geq s$. At step $t$, it calls Algorithm~\ref{alg:Op_Prod} for the distribution $\tilde{p}_t$ over the $\tilde{K}$ experts, and computes its own distribution $p_t$ over $K$ arms, with $p_{t,k}$ proportional to $\sum_{s=1}^t \tilde{p}_{t,(s,k)}$. Then it plays $p_t$, receives loss vector $\ell_t$, and feeds some modified loss vector $\tilde{\ell}_t$ and estimate vector $\tilde{m}_t$ to Algorithm~\ref{alg:Op_Prod} for update. Here, we set $\tilde{\ell}_{t,(s,k)}$ to its expected loss $\langle p_t, \ell_t\rangle$ if expert $(s,k)$ is asleep and to $\ell_{t,k}$ otherwise, while we set $\tilde{m}_{t,(s,k)}$ to $0$ if expert $(s,k)$ is asleep and to $m_{t,k} = \langle p_t, \ell_{t-1}\rangle - \ell_{t-1,k}$ otherwise. This choice allows us to relate the regret of Algorithm~\ref{alg:Op_Prod} to that of the new algorithm, which can be seen in the proof of the following theorem, given in Appendix~\ref{app:sleep} in the supplementary material.

\begin{theorem}\label{theorem:dynamic_upper_full-into}
The dynamic expected regret of the new algorithm is $\tilde{\mathcal{O}}(\sqrt{\Gamma\W \ln K} +\Gamma \ln K)$ for switching distributions and $\tilde{\mathcal{O}}(\sqrt[3]{V\W T \ln K} + \sqrt{VT \ln K})$ for drifting distributions.
\end{theorem}

\section{Lower Bounds} \label{sec:low}

We study regret lower bounds in this section. In subsection~\ref{sec:bandit}, we show that for switching distributions with $\Gamma-1 \ge 1$ switches, there is an $\Omega(\sqrt{\Gamma T})$ lower bound for bandit algorithms, even when there is no variance ($\W=0$) and there are constant loss gaps between the optimal and suboptimal arms. We also show a full-information lower bound, which almost matches our upper bound in Theorem~\ref{thm:dev}.
In subsection~\ref{sec:drift2}, we show that for drifting distributions, our upper bounds in Theorem~\ref{thm:bandit-u} and Theorem~\ref{thm:dev} are almost tight. In particular, we show that now even for full-information algorithms, a large $\sqrt[3]{T}$ dependency in the regret turns out to be unavoidable, even for small $V$ and $\W$. This provides a sharp contrast to the upper bound of our Theorem~\ref{thm:dev}, which shows that a constant regret is in fact achievable by a full-information algorithm for switching distributions with constant $\Gamma$ and $\W$. For simplicity of presentation, we will only discuss the case with $K=2$ actions, as it is not hard to extend our proofs to the general case.

\subsection{Switching Distributions} \label{sec:bandit}

In contrast to the full-information setting, the existence of switches presents a dilemma with lose-lose situation for a bandit algorithm: in order to detect any possible switch early enough, it must explore aggressively, but this has the consequence of playing suboptimal arms too often. Our lower bound construction shares some similarity with \cite{daniely2015strongly}'s Theorem 3: to fool any bandit algorithm, we will switch between two deterministic distributions, with no variance, which have mean vectors $\ell^{(1)} = (1/2, 1)^\top$ and $\ell^{(2)} = (1/2, 0)^\top$, respectively. Our result is the following.

\begin{theorem} \label{thm:bandit-low}
The worst-case expected regret of any bandit algorithm is $\Omega(\sqrt{\Gamma T})$, for $\Gamma \ge 2$.
\end{theorem}
\begin{proof}
Consider any bandit algorithm $\mathcal{A}$, and let us partition the $T$ steps into $\Gamma/2$ intervals, each consisting of $B=2T/\Gamma$ steps. Our goal is to make $\mathcal{A}$ suffer in each interval an expected regret of $\Omega(\sqrt{B})$ by switching the loss vectors at most once. As mentioned before, we will only switch between two different deterministic distributions with mean vectors: $\ell^{(1)}$ and $\ell^{(2)}$. Note that we can see these two distributions simply as two loss vectors, with $\ell^{(i)}$ having arm $i$ as the optimal arm.

In what follows, we focus on one of the intervals, and assume that we have chosen the distributions in all previous intervals. We would like to start the interval with the loss vector $\ell^{(1)}$. Let $N_2$ denote the expected number of steps $\mathcal{A}$ plays the suboptimal arm 2 in this interval if $\ell^{(1)}$ is used for the whole interval. If $N_2 \ge \sqrt{B}/2$, we can actually use $\ell^{(1)}$ for the whole interval with no switch, which makes $\mathcal{A}$ suffer an expected regret of at least $(1/2) \cdot \sqrt{B}/2 = \sqrt{B}/4$ in this interval.
Thus, it remains to consider the case with $N_2 < \sqrt{B}/2$. In this case, $\mathcal{A}$ does not explore arm 2 often enough, and we let it pay by choosing an appropriate step to switch to the other loss vector $\ell^{(2)} = (1/2, 0)^\top$, which has arm 2 as the optimal one. For this, let us divide the $B$ steps of the interval into $\sqrt{B}$ blocks, each consisting of $\sqrt{B}$ steps.
As $N_2 < \sqrt{B}/2$, there must be a block in which the expected number of steps that $\mathcal{A}$ plays arm 2 is at most $N_2 / \sqrt{B} < 1/2$. By a Markov inequality, the probability that $\mathcal{A}$ ever plays arm 2 in this block is less than $1/2$. This implies that when given the loss vector $\ell^{(1)}$ for all the steps till the end of this block, $\mathcal{A}$ never plays arm 2 in the block with probability more than $1/2$. Therefore, if we make the switch to the loss vector $\ell^{(2)}= (1/2, 0)^\top$ at the beginning of the block, then $\mathcal{A}$ with probability more than $1/2$ still never plays arm 2 and never notices the switch in this block. As arm 2 is the optimal one with respect to $\ell^{(2)}$, the expected regret of $\mathcal{A}$ in this block is more than $(1/2) \cdot (1/2) \cdot \sqrt{B} = \sqrt{B}/4$.

Now if we choose distributions in each interval as described above, then there are at most $\Gamma/2\cdot 2=\Gamma$ periods of stationary distribution in the whole horizon, and the total expected regret of $\mathcal{A}$ can be made at least $\Gamma/2\cdot\sqrt{B}/4=\Gamma/2\cdot\sqrt{2T/\Gamma}/4=\Omega(\sqrt{\Gamma T})$, which proves the theorem.
\end{proof}

For full-information algorithms, we have the following lower bound, which almost matches our upper bound in Theorem~\ref{thm:dev}. We provide the proof in Appendix~\ref{app:full-low} in the supplementary material.

\begin{theorem} \label{thm:full-low}
The worst-case expected regret of any full-information algorithm is $\Omega(\sqrt{\Gamma \W} + \Gamma)$.
\end{theorem}

\subsection{Drifting Distributions} \label{sec:drift2}

In this subsection, we show that the regret upper bounds achieved by our bandit algorithm and full-information algorithm are close to optimal by showing almost matching lower bounds. More precisely, we have the following.

\begin{theorem}
\label{thm:full_V}
The worst-case expected regret of any full-information algorithm is $\Omega(\sqrt[3]{\Lambda VT}+V)$, while that of any bandit algorithm is $\Omega(\sqrt[3]{\Lambda VT}+\sqrt{VT})$.
\end{theorem}

\begin{proof}


Let us first consider the full-information case. 
When $\Lambda T \le 32 K V^2$, we immediately have from Theorem~\ref{thm:full-low} the regret lower bound of $\Omega(\Gamma) \ge \Omega(V) \ge \Omega(\sqrt[3]{\Lambda VT} + V)$.

Thus, let us focus on the case with $\Lambda T\geq 32 K V^2$. In this case, $V \le \mathcal{O}(\sqrt[3]{\Lambda VT})$, so it suffices to prove a lower bound of $\Omega(\sqrt[3]{\Lambda VT})$.
Fix any full-information algorithm $\mathcal{A}$, and we will show the existence of a sequence of loss distributions for $\mathcal{A}$ to suffer such an expected regret. Following \cite{gur2014stochastic},
we divide the time steps into $T/B$ intervals of length $B$, and we set $B=\sqrt[3]{\Lambda T / (32 KV^2)} \ge 1$.
For each interval, we will pick some arm $i$ as the optimal one, and give it some loss distribution $\mathcal{P}$, while other arms are sub-optimal and all have some loss distribution $\mathcal{Q}$.
We need $\mathcal{P}$ and $\mathcal{Q}$ to satisfy the following three conditions: (a) $\mathcal{P}$'s mean is smaller than $\mathcal{Q}$'s by $\epsilon$, (b) their variances are at most $\sigma^2$, and (c) their KL divergence satisfies $(\ln 2) \mathrm{KL}(\mathcal{Q}, \mathcal{P}) \le \epsilon^2/\sigma^2$, for some $\epsilon, \sigma \in (0,1)$ to be specified later. Their existence is guaranteed by the following, which we prove in Appendix~\ref{app:sigma1} in the supplementary material.
\begin{lemma}
\label{lemma:sigma1}
For any $0\leq\sigma\leq 1/2$ and $0\leq\epsilon\leq \sigma/\sqrt{2}$, there exist distributions $\mathcal{P}$ and $\mathcal{Q}$ satisfying the three conditions above.
\end{lemma}

Let $\mathcal{D}_i$ denote the joint distribution of such $K$ distributions, with arm $i$ being the optimal one, and we will use the same $\mathcal{D}_i$ for all the steps in an interval.
We will show that for any interval, there is some $i$ such that using $\mathcal{D}_i$ this way can make algorithm $\mathcal{A}$ suffer a large expected regret in the interval, conditioned on the distributions chosen for previous intervals. Before showing that, note that when we choose distributions in this way, their total variance is at most $TK\sigma^2$ while their total drift is at most $(T/B)\epsilon$. To have them bounded by $\Lambda$ and $V$ respectively, we choose $\sigma= \sqrt{\Lambda/(4KT)}$ and $\epsilon = VB/T$, which satisfy the condition of Lemma~\ref{lemma:sigma1}, with our choice of $B$.

To find the distributions, we deal with the intervals one by one. Consider any interval, and assume that the distributions for previous intervals have been chosen. Let $N_i$ denote the number of steps $\mathcal{A}$ plays arm $i$ in this interval, and let $\mathbb{E}_i[N_i]$ denote its expectation when $\mathcal{D}_i$ is used for every step of the interval, conditioned on the distributions of previous intervals. One can bound this conditional expectation in terms of a related one, denoted as $\mathbb{E}_{\textit{unif}}[N_i]$, when every arm has the distribution $\mathcal{Q}$ for every step of the interval, again conditioned on the distributions of previous intervals. Specifically, using an almost identical argument to that in \cite[proof of Theorem A.2.]{auer2002nonstochastic}, one can show that
\begin{equation}\label{eq:N_i}
\mathbb{E}_i\left[ N_i \right] \leq\mathbb{E}_{\textit{unif}}\left[ N_i \right] + \frac{B}{2}\sqrt{B(2\ln 2) \cdot \mathrm{KL}(\mathcal{Q}, \mathcal{P})}.\footnote{Note that inside the square root, we use $B$ instead of $\mathbb{E}_{\textit{unif}}[N_i]$ as in \cite{auer2002nonstochastic}. This is because in their bandit setting, $N_i$ is the number of steps when arm $i$ is sampled and has its information revealed to the learner, while in our full-information case, information about arm $i$ is revealed in every step and there are at most $B$ steps.}
\end{equation}
According to Lemma~\ref{lemma:sigma1} and our choice of parameters, we have $B (2\ln 2) \cdot \mathrm{KL}(\mathcal{Q}, \mathcal{P}) \le 2B \cdot (\epsilon^2/\sigma^2) \le 1/4$. Summing both sides of (\ref{eq:N_i}) over arm $i$, and using the fact that $\sum_{i}\mathbb{E}_{\textit{unif}}\left[ N_i \right] = B $, we get
$\sum_{i} \mathbb{E}_i\left[ N_i \right] \leq B+ BK/4,$
which implies the existence of some $i$ such that $\mathbb{E}_i\left[ N_i \right] \leq B/K + B/4 \le (3/4)B.$
Therefore, if we choose this distribution $\mathcal{D}_i$, the conditional expected regret of algorithm $\mathcal{A}$ in this interval is at least
$\epsilon ( B- \mathbb{E}_i [N_i]) \ge \epsilon B/4.$

By choosing distributions inductively in this way, we can make $\mathcal{A}$ suffer a total expected regret of at least
$(T/B) \cdot (\epsilon B/4) \ge \Omega(\sqrt[3]{\W VT}).$
This completes the proof for the full-information case.

Next, let us consider the bandit case. From Theorem~\ref{thm:bandit-low}, we immediately have a lower bound of $\Omega(\sqrt{\Gamma T}) \ge \Omega(\sqrt{V T})$, which implies the required bound when $\sqrt{V T} \ge \sqrt[3]{\W VT}$. When $\sqrt{V T} \le \sqrt[3]{\W VT}$, we have $V \le \W^2/T$ which implies that $V \le \sqrt[3]{\W VT}$, and we can then use the full-information bound of $\Omega(\sqrt[3]{\W VT})$ just proved before. This completes the proof of the theorem.
\end{proof}

\medskip

\bibliographystyle{plain}
\bibliography{Bib1}

\newpage

\appendix

\section{Summary and Comparison of Regret Bounds} \label{app:list}

In Table~\ref{tab:comp}, we list the bounds we derive as well as those from previous works. As can be seen, we have completed a basic picture of the regret landscape characterized by the parameters $\Gamma, V, \W$, and $T$, in both the full-information and bandit settings. Note that we did not provide a bandit upper bound for switching distributions because a good upper bound is already shown in \cite{garivier2011upper}. Although that upper bound does not depend on $\W$ which may lead one to wonder if a better bound is possible with a smaller $\W$, we show that it is in fact impossible by providing a matching lower bound using distributions with $\W=0$. Let us remark that although an $\Omega(\sqrt{\Gamma T})$ lower bound was also given in \cite{garivier2011upper}, it was achieved by distributions with $\W \ge \Omega(T)$. Moreover, the upper bounds of previous works, such as \cite{garivier2011upper,gur2014stochastic,besbes2015non}, do not take $\W$ into account, so their algorithms do not seem to produce our regret bounds in terms of $\W$. The full-information algorithms of \cite{luo2015achieving,Gaillard2014,koolen2015second} can produce a regret bound of the form $\mathcal{O}(\sqrt{\Gamma L^*})$ for switching distributions, with the quantity $L^*$ being the smallest accumulated loss by an expert sequence that changes at most $\Gamma$ times. 
Finally, in the full-information setting, \cite{jadbabaie2015online} provides a regret upper bound which resembles that of ours for drifting distributions but using slightly different definition of parameters.



\begin{table*}[!h]
\begin{minipage}{1.0\textwidth}
\caption{Summary of regret bounds for switching and drifting distributions \label{tab:comp}}
\small
\begin{center}
    \begin{tabular}{@{\extracolsep{3pt}}K{1.4cm}K{1.4cm}K{1.6cm}K{1.6cm}K{0.5cm}K{0.5cm}K{0.5cm}K{0.5cm} K{0.5cm}K{0.5cm}}
      \multicolumn{2}{c}{Setting} & \multicolumn{2}{c}{This paper}  &  \multicolumn{2}{c}{\cite{luo2015achieving,Gaillard2014,koolen2015second}} & \multicolumn{2}{c}{\cite{garivier2011upper}}& \multicolumn{2}{c}{\cite{gur2014stochastic,besbes2015non}}\\ \cmidrule{1-2} \cmidrule{3-4}\cmidrule{5-6}\cmidrule{7-8}\cmidrule{9-10}
   \scalebox{0.9}{Feedback} & \scalebox{0.9}{Distributions} & \scalebox{0.9}{Upper}& \scalebox{0.9}{Lower} & \scalebox{0.9}{Upper} &\scalebox{0.9}{Lower} &\scalebox{0.9}{Upper} &\scalebox{0.9}{Lower} & \scalebox{0.9}{Upper} & \scalebox{0.9}{Lower}\\ \cmidrule{1-2} \cmidrule{3-4} \cmidrule{5-6} \cmidrule{7-8} \cmidrule{9-10}

\scalebox{0.9}{Full-info} & \scalebox{0.9}{Switching} & \scalebox{0.8}{$\sqrt{\Gamma\W}+\Gamma$} & \scalebox{0.8}{$\sqrt{\Gamma\W}+\Gamma$} & \scalebox{0.8}{$\sqrt{\Gamma L^{^*}}$} & - &- &- & - & -\\
\scalebox{0.9}{Full-info} & \scalebox{0.9}{Drifting} & \scalebox{0.8}{$\sqrt[3]{V\W T}+V$} & \scalebox{0.8}{$\sqrt[3]{V\W T}+V$} & - & - & - & - & \scalebox{0.8}{$\sqrt[3]{V T^2}$} & \scalebox{0.8}{$\sqrt[3]{V T^2}$} \\ \cmidrule{1-2} \cmidrule{3-4} \cmidrule{5-6} \cmidrule{7-8} \cmidrule{9-10}
\scalebox{0.9}{Bandit} & \scalebox{0.9}{Switching} & - & \scalebox{0.8}{$\sqrt{\Gamma T}$} & - & - & \scalebox{0.8}{$\sqrt{\Gamma T}$} & \scalebox{0.8}{$\sqrt{\Gamma T}$} & - &-\\
\scalebox{0.9}{Bandit} & \scalebox{0.9}{Drifting} & \scalebox{0.8}{$\sqrt[3]{V \W T}+ \sqrt{VT}$} & \scalebox{0.8}{$\sqrt[3]{V \W T}+ \sqrt{VT}$} & - & - & - & - & \scalebox{0.8}{$\sqrt[3]{V T^2}$} & \scalebox{0.8}{$\sqrt[3]{V T^2}$}\\ \cmidrule{1-2} \cmidrule{3-4} \cmidrule{5-6} \cmidrule{7-8} \cmidrule{9-10}
    \end{tabular}
\end{center}
\end{minipage}
\end{table*}




\section{Proof of Theorem~\ref{thm:bandit-u}} \label{app:bandit-u}

To prove the theorem, we rely on the following key lemma, which we prove in Subsection~\ref{app:reg_int}.

\begin{lemma}
\label{lemma:reg_int}
Consider a time interval $\I$. Let $V_{\I}=\sum_{t\in \I} \|\mu_{t}-\mu_{t-1}\|_\infty$ and $\W_{\I}=\sum_{t\in \I} \ex{\|\ell_t - \mu_{t}\|_2^2}$ be the drift and variance, respectively, in $\I$. Then the expected regret of Algorithm~\ref{alg:bandit} in $\I$ is at most
$\mathcal{\tilde{O}} \left(K\sqrt{\W_{\I}} + K + \lvert\I\rvert V_{\I} \right).$
\end{lemma}

By applying Lemma~\ref{lemma:reg_int} on each interval, we can bound the total regret of Algorithm~\ref{alg:bandit} by
$$\sum_{\gamma=1}^{T/B} \mathcal{\tilde{O}} \left(K \sqrt{\W_{\I_\gamma}} + K + B V_{\I_\gamma} \right) \le \mathcal{\tilde{O}}\left(K \sqrt{T\W/B} + K T/B  + BV \right)$$
from the Cauchy--Schwarz inequality.
This gives the bound of $\tilde{O}(\sqrt[3]{K^2 \Lambda VT})$ with $B=\sqrt[3]{K^2 \Lambda T/V^2}$ when $K \Lambda^2 \ge TV$, while the bound becomes $\mathcal{\tilde{O}}(\sqrt{KVT})$ with $B=\sqrt{KT/V}$ otherwise. Combining these two bounds together gives the regret bound of the theorem.

\subsection{Proof of Lemma~\ref{lemma:reg_int}} \label{app:reg_int}

Let us first consider a time step $t$ and bound the regret of our algorithm at that step, which is $\mu_{t,a_t}-\mu_{t,u_t^*}$, where $a_t$ denotes the arm chosen by our algorithm and $u_t^*$ is the best arm. Let $\bar{\mu}_{t,i} = \sum_{s \in S_{t,i}} \frac{\mu_{s,i}}{|S_{t,i}|}$, which is the expected value of $\hat{\mu}_{t,i}$. Note that when arm $a_t$ is pulled, we have $\hat{\mu}_{t,a_t}-\lambda_{t,a_t}\le \hat{\mu}_{t,u_t^*}-\lambda_{t,u_t^*}.$ This implies
$$\bar{\mu}_{t,a_t}-\bar{\mu}_{t,u_t^*} \le 2\lambda_{t,a_t}$$ when $|\bar{\mu}_{t,i}-\hat{\mu}_{t,i}| \le \lambda_{t,i}$ for every arm $i$, which happens with probability $1-\delta K$
by Theorem \ref{thm:emp_bernstein} and a union bound. Using the fact that $|\bar{\mu}_{t,i} - \mu_{t,i}| \le V_{\I}$ for any arm $i$, we have $$\mu_{t,a_t}-\mu_{t,u_t^*} \le \bar{\mu}_{t,a_t}-\bar{\mu}_{t,u_t^*} + 2 V_{\I} \le  2\lambda_{t,a_t} + 2 V_{\I}$$
with probability $1-\delta K$.

As a result, by summing over $t \in \I$, the total regret of our algorithm is at most
\begin{equation}\label{eq:lambda}
2\sum_{t\in \I} \lambda_{t,a_t} + 2 V_{\I} |\I| + \delta K \lvert\I\rvert,
\end{equation}
with the last term being at most $1$ for $\delta \le 1/(KT)$. It remains to bound the first term above, for which we rely on the following lemma.
\begin{lemma} \label{lem:lambda}
For any time step $t$ and any arm $i$ such that $|S_{t,i}| > 0$,
$$\lambda_{t,i} \le \mathcal{\tilde{O}}\left(\frac{1}{|S_{t,i}|} \left(\sqrt{\sum_{r \in S_{t,i}} \hat{\sigma}_{r,i}^2} + 1 \right) + V_{\I}\right), \;\mbox{ where } \hat{\sigma}_{r,i}^2 = (\ell_{r,i} - \mu_{r,i})^2.$$
\end{lemma}
Before proving the lemma, let us first use it to bound the first term in (\ref{eq:lambda}). Let us divide $\I$ into two parts: $\I' = \{t \in \I: |S_{t,a_t}| \ge 2 \}$ and $\I \setminus \I'$. Then $\sum_{t\in \I'} \lambda_{t,a_t}$ is at most
$$\mathcal{\tilde{O}}\left(\sum_{t\in \I'} \frac{1}{|S_{t,a_t}|} \left(\sqrt{\sum_{r \in \I'} \hat{\sigma}_{r,a_r}^2} + 1\right) + |\I'| V_{\I'} \right) \le \mathcal{\tilde{O}}\left(K \sqrt{\sum_{r \in \I'} \hat{\sigma}_{r,a_r}^2} + K + |\I| V_{\I} \right),$$
since $\sum_{t \in \I'} \frac{1}{|S_{t,a_t}|} \le K \sum_{t \in \I} \frac{1}{t} \le K \ln |\I|$, and recall that we use $\mathcal{\tilde{O}}(\cdot)$ to hide logarithmic factors. Then by combining all the bounds together, we have
$$\sum_{t\in \I} \lambda_{t,a_t} \le 2K + \sum_{t\in \I'} \lambda_{t,a_t} \le \mathcal{\tilde{O}}\left(K\sqrt{\sum_{t\in \I} \hat{\sigma}_{t,a_t}^2}+ K + |\I| V_{\I} \right).$$
Substituting this into the bound in (\ref{eq:lambda}) and taking its expectation, we can conclude that the expected regret of our algorithm is at most
$$\mathcal{\tilde{O}}\left(K\sqrt{\sum_{t\in \I}\sum_{i \in [K]} \ex{\hat{\sigma}_{r,i}^2}} + K + |\I| V_{\I} \right) = \mathcal{\tilde{O}}\left(K\sqrt{\W_\I} + K + |\I| V_{\I} \right),$$
by Jensen's inequality and the definition of $\W_\I$. This gives the regret bound claimed by Lemma~\ref{lemma:reg_int}. Thus, to complete the proof, it remains to prove Lemma~\ref{lem:lambda}, which we do next.

\begin{proof}[Proof of Lemma~\ref{lem:lambda}]
Consider any $t$ and $i$, and for ease of presentation, let us drop the indices involving $t$ and $i$. To bound $\lambda = \rho(|S|, \hat{\W}, \delta)$, defined in Theorem~\ref{thm:emp_bernstein}, let us first bound $\hat{\W} = \sum_{r,s \in S} (\ell_r - \ell_s)^2 / (|S|(|S|-1))$. Note that each $(\ell_r - \ell_s)^2$ can be expressed as
\begin{eqnarray}
((\ell_r - \mu_r) - (\ell_s - \mu_s) + (\mu_r - \mu_s))^2
&\le& 3 \left((\ell_r - \mu_r)^2 + (\ell_s - \mu_s)^2 + (\mu_r - \mu_s)^2\right) \label{eq:l-mu}\\
&\le& 3 \left(\hat{\sigma}_r^2 + \hat{\sigma}_s^2 + V^2 \right),\nonumber
\end{eqnarray}
using the Cauchy--Schwarz inequality as well as the definition $(\ell_r - \mu_r)^2 = \hat{\sigma}_r^2$ and the fact $(\mu_r - \mu_s)^2 \le V^2$. Thus, we have
$$\hat{\W} = \sum_{r,s \in S} \frac{(\ell_r - \ell_s)^2}{|S|(|S|-1)} \le \mathcal{O}\left(\sum_{r \in S} \frac{\hat{\sigma}_r^2}{|S|}  + V^2 \right).$$
By plugging this bound into the definition of $\lambda$ and using $\mathcal{\tilde{O}}(\cdot)$ to hide logarithmic factors, we have
$$\lambda \le \mathcal{\tilde{O}}\left(\sqrt{\sum_{r \in S} \frac{\hat{\sigma}_r^2}{|S|^2} + V^2} + \frac{1}{|S|} \right) \le \mathcal{\tilde{O}}\left(\frac{1}{|S|} \left(\sqrt{\sum_{r \in S} \hat{\sigma}_r^2} +1 \right) + V \right),$$
since $\sqrt{a+b}\le \sqrt{a}+\sqrt{b}$ for any $a,b\ge 0$. This proves the lemma.
\end{proof}

\section{Proof of Theorem~\ref{thm:dev}} \label{app:dev}

Recall that our Algorithm~\ref{alg:dev} is based on that of \cite{ChiangYLMLJZ12} which is known to have a good regret bound against any offline comparator which does not switch arms. As we consider the dynamic regret here, we need to extend the bound to work against offline comparators which can switch arms. The following lemma provides such a bound.

\begin{lemma} \label{lem:dev}
With the choice of $\eta=\sqrt{N/(\W+KV)}$, the regret of Algorithm~\ref{alg:dev}, against an offline comparator switching arms less than $N$ times, is at most $\mathcal{O}(\sqrt{N(\W+KV)})$.
\end{lemma}
\begin{proof}
Consider any offline comparator which switches arms $N-1$ times, say at $t_1, \dots, t_{N-1}$. More specifically, assuming $t_0=1$ and $t_N=T+1$, the arm it plays at step $t \in [t_n, t_{n+1})$ remains the same as that at step $t_n$, for any $0 \le n \le N-1$.
Let $\pi_t$ denote the characteristic vector of the arm it plays at step $t$. Then when compared against it, the expected regret of our algorithm at step $t$ with respect to a realized loss vector $\ell_t$ (sampled from the distribution) is
$$\sum_{i=1}^K \ell_{t,i} \left(\hat{x}_{t,i} - \pi_{t,i}\right) = \langle \ell_t, \hat{x}_t - \pi_t \rangle,$$
which according to \cite{ChiangYLMLJZ12} is at most
$$\eta \|\ell_t - \ell_{t-1}\|_2^2 + \frac{1}{2\eta}\left(\|\pi_t - \hat{x}_t\|_2^2 - \|\pi_t - \hat{x}_{t+1}\|_2^2\right).\footnote{Although \cite{ChiangYLMLJZ12} considered the case with identical $\pi_t = \pi$ for every $t$, it is straightforward to see this from their Lemma~5 and their proof of Theorem~8.}$$
To bound the total expected regret with respect to these realized loss vectors, let us sum the above over $t$ and observe that the second term above can telescope within each interval and have only $N$ terms remaining at the boundaries between intervals. More precisely,
we have
\begin{eqnarray*}
\sum_{t=1}^T \left(\|\pi_t - \hat{x}_t\|_2^2 - \|\pi_t - \hat{x}_{t+1}\|_2^2\right) &=& \sum_{n=0}^{N-1} \sum_{t=t_n}^{t_{n+1}-1} \left(\|\pi_{t_n} - \hat{x}_t\|_2^2 - \|\pi_{t_n} - \hat{x}_{t+1}\|_2^2\right)\\ &\le& \sum_{n=0}^{N-1} \|\pi_{t_n} - \hat{x}_{t_n}\|_2^2,
\end{eqnarray*}
which is at most $2N$ as $\|\pi_n - \hat{x}_{t_n}\|_2^2 \le 2$. Finally, by taking an additional expectation over the randomness of these loss vectors, the resulting expected regret of Algorithm~\ref{alg:dev} can be bounded from above by
$$\eta \sum_{t=1}^T \ex{\|\ell_t - \ell_{t-1}\|_2^2} + \frac{N}{\eta}.$$
To bound the first sum above, note that using the notation $\sigma_\tau^2 = \ex{\|\ell_\tau\|_2^2} - \|\mu_\tau\|_2^2$, we have
\begin{eqnarray}
\ex{\|\ell_t - \ell_{t-1}\|_2^2} &=& \ex{\|\ell_t\|_2^2} + \ex{\|\ell_{t-1}\|_2^2} - 2 \langle \mu_t, \mu_{t-1}\rangle \nonumber\\
&=& \sigma_t^2 + \sigma_{t-1}^2 + \|\mu_t\|_2^2 + \|\mu_{t-1}\|_2^2 - 2 \langle \mu_t, \mu_{t-1}\rangle \nonumber\\ &=& \sigma_t^2 + \sigma_{t-1}^2 + \|\mu_t-\mu_{t-1}\|_2^2. \label{eq:l-l}
\end{eqnarray}
This implies that the expected regret of Algorithm~\ref{alg:dev} is at most
$$2 \eta \sum_{t=0}^T \sigma_t^2 + \eta \sum_{t=1}^T \|\mu_t-\mu_{t-1}\|_2^2 + \frac{N}{\eta} \le 2 \eta \W + \eta K V + \frac{N}{\eta},$$
as $\sigma_t^2 = \ex{\|\ell_t - \mu_t\|_2^2}$ and $\sum_{t=0}^T \ex{\|\ell_t - \mu_t\|_2^2} \le \W$, using the convention that $\sigma_0=0$.
Then the lemma follows by choosing $\eta = \sqrt{N/(\W+KV)}$.
\end{proof}

For switching distributions parameterized by $\Gamma$, we have $V \le \Gamma$, and by choosing $N=\Gamma$ and $\eta= \sqrt{\Gamma/(\W+K\Gamma)}$, Lemma~\ref{lem:dev} provides a regret bound of $\mathcal{O}(\sqrt{\Gamma(\W+K\Gamma)}) \le \mathcal{O}(\sqrt{\Gamma \W} + \sqrt{K} \Gamma)$ for Algorithm~\ref{alg:dev}.



\section{Proof of Theorem~\ref{thm:full_info_given_V}} \label{app:free}


Let us first focus on a single interval $\I= [t_1, t_2]$ of time steps, and consider the simpler case of static regret, measured against a single fixed arm.
Recall that our learning rates are set as
$$\eta_t = 1/ \sqrt{4\sum_{\tau=t_1}^{t-1} \lVert \ell_\tau -\ell_{\tau-1}\rVert ^2_2}$$
for $t \in \I$, with $\eta_{t_1} = \infty$ initially for every interval.
Note that our learning rates are set similarly to those used in \cite{jadbabaie2015online}. The main difference is that we allow them to be $\infty$, which prevents us from applying their result directly. Still, similarly to \cite[Lemma 1]{jadbabaie2015online}, we can have the following guarantee.

\begin{lemma}
Using the learning rates given above, the regret of Algorithm \ref{alg:dev} in any interval $\I = [t_1, t_2]$ with respect to any fixed arm is at most
$\mathcal{O}\left(\sqrt{ \sum_{t \in \I} \lVert \ell_t-\ell_{t-1}\rVert _2^2}\right).$
\end{lemma}
\begin{proof}
Let $\rho$ be the first time step in $\I$ such that $\lVert \ell_{\rho}-\ell_{\rho-1} \rVert_2 \neq 0$, so that for any step $t < \rho$ in $\I$, we have $\eta_t =\infty$ and
$\ell_t=\ell_{t-1}$. This implies that for any step $t < \rho$ in $\I$, what Algorithm~\ref{alg:dev} plays is $\hat{x}_t=\arg\min_{x\in\X} \langle \ell_{t-1}, x\rangle=\arg\min_{x\in\X} \langle \ell_t, x\rangle$, which is optimal and has no positive regret.
Therefore, it suffices to bound the regret accumulated from step $\rho$ to step $t_2$, which we can do by applying \cite[Lemma 1]{jadbabaie2015online}. To see this, note that the bound there works for any arbitrary starting point $\hat{x}_\rho$ (here computed according to $\eta_\rho = \infty$), and after that we indeed have $\eta_t < \infty$ for $t>\rho$ in $\I$. Then according to \cite[Lemma 1]{jadbabaie2015online}, we can upper-bound the regret accumulated from step $\rho$ to step $t_2$ by
$$\sum_{t=\rho}^{t_2} \frac{\eta_{t+1}}{2}\lVert \ell_t-\ell_{t-1}\rVert_2^2+\frac{8}{\eta_{t_2+1}}
= \mathcal{O}\left(\sqrt{\sum_{t=t_1}^{t_2} \lVert\ell_t- \ell_{t-1}\rVert_2^2}\right).$$
This proves the lemma.
\end{proof}

The regret bound above is stated in terms of any specific realization of loss vectors. By taking the expectation over their randomness, we can upper-bound the expected regret by
$$\mathcal{O}\left(\ex{\sqrt{\sum_{t=t_1}^{t_2} \lVert\ell_t- \ell_{t-1}\rVert_2^2}}\right) \le \mathcal{O}\left(\sqrt{\sum_{t=t_1}^{t_2} \ex{\lVert\ell_t- \ell_{t-1}\rVert_2^2}}\right),$$
using Jensen's inequality, which according to \eqref{eq:l-l} in the proof of Lemma~\ref{lem:dev} is at most
$$\mathcal{O}\left(\sqrt{\sum_{t=t_1-1}^{t_2} \sigma_t^2 + \sum_{t=t_1}^{t_2} \lVert\mu_t- \mu_{t-1}\rVert_2^2}\right) \le \mathcal{O}\left(\sqrt{\sum_{t=t_1-1}^{t_2} \sigma_t^2} + \sqrt{K} \sum_{t=t_1}^{t_2} \lVert\mu_t- \mu_{t-1}\rVert_\infty \right),$$
using the fact that $\sqrt{a + b} \le \sqrt{a} + \sqrt{b}$ for $a,b\ge 0$ and $\lVert\mu_t- \mu_{t-1}\rVert_2^2 \le K \lVert\mu_t- \mu_{t-1}\rVert_\infty^2$.

By summing such bounds for the $T/B$ intervals, we can upper bound the total regret by $$\mathcal{O}\left(\sqrt{(T/B) \sum_{t=0}^{T} \sigma_t^2} + \sqrt{K} \sum_{t=1}^{T} \lVert\mu_t- \mu_{t-1}\rVert_\infty \right) \le \mathcal{O}(\sqrt{(T/B) \W} + \sqrt{K} V)$$
using the Cauchy–Schwarz inequality and the definition of $\W$ and $V$.
Note that the regret above works against the best offline comparator which plays a fixed arm in each interval (possibly different arms in different intervals),
while our goal is to compete with any fully dynamic comparator which can switch arms at every step. Therefore, we would like to bound the difference between the total losses of these two comparators. This can be done similarly to that in the proof of Lemma~\ref{lemma:reg_int}. More precisely, let us focus on one interval $\I$, and we claim that the difference between their losses in $\I$ is at most $2B V_\I$, where $V_\I = \sum_{t \in \I} \| \mu_t - \mu_{t-1}\|_\infty$ is the total drift in $\I$. To see this, suppose arm $k$ is the best fixed arm in $\I$, which implies that it must the best arm at some step $s \in \I$. Then for any step $t \in \I$, it has $\mu_{t,k} \le \mu_{s,k} + V_\I$ while the best arm at step $t$, denoted as $u_t*$, has $\mu_{t,u_t*} \ge \mu_{s,u_t*} - V_\I \ge \mu_{s,k} - V_\I$. Therefore, we have
$$\sum_{t \in \I} \left(\mu_{t,k} - \mu_{t, u_t*}\right) \le \sum_{t \in \I} \left(\left(\mu_{s,k}+V_\I\right) - \left(\mu_{s,k}-V_\I\right)\right) \le 2B V_\I.$$
Summing over the intervals, we can then upper-bound their total difference by $2BV$.

As a result, we can conclude that the total regret of our algorithm against any fully dynamic comparator is at most
$\mathcal{O}(\sqrt{(T/B)\W}+ \sqrt{K}V +BV).$
This bound is at most $\mathcal{O}(\sqrt[3]{\W V T}+\sqrt{K}V)$ by setting $B=\sqrt[3]{\W T/V^2}$ when $\W T> V^2$ and at most $\mathcal{O}(\sqrt{K}V)$ by setting $B=1$ otherwise. Thus, we have the required bound in each case, which proves the theorem.


\section{Proof of Theorem~\ref{thm_RTk_bound}} \label{app:RTK}

Let $R_{T,k}$ denote the regret of our Optimistic-Adapt-ML-Prod algorithm (Algorithm~\ref{alg:Op_Prod}) with respect to expert $k$, which is defined as
$$R_{T,k} = \sum_{t=1}^T r_{t,k} = \sum_{t=1}^T (\langle p_t, \ell_t\rangle - \ell_{t,k}).$$
First, as in \cite[Theorem 3]{Gaillard2014}, we can bound the regret $R_{T,k}$ in the following form.

\begin{lemma} \label{lemma_regret}
For any $k$, $R_{t,k}$ is at most
\begin{equation}\label{eqn:regret}
\frac{1}{\eta_{0,k}}\ln \frac{1}{w_{0,k}}+\sum_{t=1}^T\eta_{t-1,k}(r_{t,k}-m_{t,k})^2+\frac{1}{\eta_{T,k}}\ln \left(1+\frac{1}{e}\sum_{k^{\prime}=1}^K\sum_{t=1}^T \left(\frac{\eta_{t-1,k^{\prime}}}{\eta_{t,k^{\prime}}}-1\right)\right).
\end{equation}
\end{lemma}

\begin{proof}
Following \cite{Gaillard2014} we derive \eqref{eqn:regret} by bounding $\ln W_T$ from above and from below. For the lower bound, one can use a similar inductive proof as in \cite{Gaillard2014} to show that for every expert $k\in [K]$,
$$\ln W_T \geq \ln w_{T,k}= \frac{\eta_{T,k}}{\eta_{0,k}}\ln w_{0,k} + \eta_{T,k} \sum_{t=1}^T \left(r_{t,k}-\eta_{t-1,k}(r_{t,k}-m_{t,k})^2\right).$$
For the upper bound, some care is needed as we play $\tilde{w}_t$ instead of $w_t$ at step $t$. As in \cite{Gaillard2014}, we would like to upper-bound each $W_{t+1}$ in terms of $W_t$. For this, consider any $k$, and note that
$$w_{t+1,k} \le (w_{t+1,k})^{\frac{\eta_{t,k}}{\eta_{t+1,k}}} + \frac{1}{e}\left(\frac{\eta_{t,k}}{\eta_{t+1,k}}-1\right),$$
using the fact that $x\leq x^\alpha +(\alpha-1)/e$ for any $x>0$ and $\alpha\geq 1$ (see \cite[Lemma 13]{Gaillard2014}). The first term on the right-hand side above can be bounded as
\begin{eqnarray*}
(w_{t+1,k})^{\frac{\eta_{t,k}}{\eta_{t+1,k}}} &=& w_{t,k}\exp\left(\eta_{t,k}r_{t+1,k}-\eta_{t,k}^2(r_{t+1,k}-m_{t+1,k})^2\right) \\
&=& \tilde{w}_{t,k} \exp\left(\eta_{t,k}(r_{t+1,k}-m_{t+1,k})- \eta_{t,k}^2(r_{t+1,k}-m_{t+1,k})^2\right)\\
&\le& \tilde{w}_{t,k}(1 + \eta_{t,k}(r_{t+1,k}-m_{t+1,k})),
\end{eqnarray*}
using the fact that $\exp(x-x^2) \le 1+x$ for any $x\geq -1/2$. Summing the above over $k$, we obtain
\begin{eqnarray*}
\sum_{k=1}^K (w_{t+1,k})^{\frac{\eta_{t,k}}{\eta_{t+1,k}}} &\le& \sum_{k=1}^K \tilde{w}_{t,k} + \sum_{k=1}^K \tilde{w}_{t,k} \eta_{t,k} r_{t+1,k} - \sum_{k=1}^K \tilde{w}_{t,k} \eta_{t,k} m_{t+1,k}\\
&\le& \sum_{k=1}^K \tilde{w}_{t,k} \exp(-\eta_{t,k} m_{t+1,k}) + \sum_{k=1}^K \tilde{w}_{t,k} \eta_{t,k} r_{t+1,k},
\end{eqnarray*}
using the fact that $1-x \le \exp(-x)$ for any $x$. The first sum equals $\sum_{k=1}^K w_{t,k} = W_t$ by definition. The second sum above equals to zero because $\tilde{w}_{t,k} \eta_{t,k} \propto p_{t+1,k}$ so that the sum is a multiple of $\sum_k p_{t+1,k} r_{t+1,k} = \sum_k p_{t+1,k} (\langle p_{t+1}, \ell_{t+1} \rangle - \ell_{t+1,k}) =0$.
As a result, we have
$$W_{t+1} = \sum_{k=1}^K w_{t+1,k} \le W_t + \sum_{k=1}^K \frac{1}{e}\left(\frac{\eta_{t,k}}{\eta_{t+1,k}}-1\right)$$
for any $t$. This implies that
$$W_T\leq 1 +\frac{1}{e}\sum_{k=1}^K\sum_{t=1}^T\left(\frac{\eta_{t-1,k}}{\eta_{t,k}}-1\right),$$
as $W_0 =1$. Connecting the upper bound and the lower bound for $\ln W_T$, we have
\begin{equation*}
\sum_{t=1}^T r_{t,k}\leq \frac{1}{\eta_{0,k}}\ln \frac{1}{w_{0,k}}+\sum_{t=1}^T\eta_{t-1,k}(r_{t,k}-m_{t,k})^2+ \frac{1}{\eta_{T,k}}\ln \left(1+\frac{1}{e}\sum_{k=1}^K\sum_{t=1}^T\left(\frac{\eta_{t-1,k}}{\eta_{t,k}}-1\right)\right),
\end{equation*}
which proves the lemma.
\end{proof}

Note that as in \cite[Theorem 3]{Gaillard2014}, the third term in \eqref{eqn:regret} is the price paid for using adaptive learning rates, while our second term is different, which depends on $r_{t,k}-m_{t,k}$ instead of $r_{t,k}$. Next, similarly to \cite[Corollary 4]{Gaillard2014}, by using the the prior $w_0=(1/K,\dots,1/K)$ and the specific learning rates defined in Eq. \eqref{eq:eta_t}, one can show that
$$R_{T,k} \le \frac{C_{T,K}}{\sqrt{\ln K}}\sqrt{1+\left(\sum_{t=1}^T (r_{t,k}-m_{t,k})^2\right)}+D_{T,K},$$
where $C_{T,K}=\ln K+\ln (1+\frac{K}{e}(1+\ln (T+1)))$ and $D_{T,K}=\frac{1}{4}(\ln K +\ln (1+\frac{K}{e} (1+ \ln (T+1)))) +2\sqrt{\ln K}+16 \ln K$.
We omit the proof because it is almost identical to that of \cite[Proof of Corollary 4]{Gaillard2014}, except we have $r_{t,k}-m_{t,k}$ in place of $r_{t,k}$ and we rely on the condition $\lvert{r_{t,k}-m_{t,k}}\rvert \leq 2$ instead of $\lvert{r_{t,k}}\rvert\leq 1$. Using the notation $\hat{\mathcal{O}}(\cdot)$ to hide the $\ln \ln T$ factor, we obtain
\begin{equation} \label{eq:RTK}
R_{T,k} \le \hat{\mathcal{O}}\left(\sqrt{\sum_{t=1}^T(r_{t,k}-m_{t,k})^2 \ln K} +\ln K \right).
\end{equation}
Finally, since
$$r_{t,k}-m_{t,k} = (\langle p_t, \ell_t \rangle - \ell_{t,k}) - (\langle p_t, \ell_{t-1} \rangle - \ell_{t-1,k}) = \langle p_t, \ell_t - \ell_{t-1} \rangle - (\ell_{t,k} - \ell_{t-1,k}),$$
we have $(r_{t,k}-m_{t,k})^2 \le (2 \|\ell_t - \ell_{t-1}\|_\infty)^2$. Substituting this into the bound in (\ref{eq:RTK}) completes the proof of the theorem.

\section{Proof of Theorem~\ref{theorem:dynamic_upper_full-into}} \label{app:sleep}

\begin{algorithm}[t]
\caption{Modified Optimistic-Adapt-ML-Prod \label{alg:sleep}}
\begin{algorithmic}
\FOR{$t=1,2,\dots,T$}
    \STATE Call Algorithm~\ref{alg:Op_Prod} on $\tilde{K}$ sleeping experts to obtain its distribution $\tilde{p}_t$ at step $t$.
    \STATE Play the distribution $p_t$, with $p_{t,k} \propto \sum_{s=1}^t \tilde{p}_{t,(s,k)}$, and then receive loss vector $\ell_t$.
    \STATE Compute $\tilde{\ell}_{t,(s,k)} = \mathds{1}_{s>t} \cdot \langle p_t, \ell_t\rangle + \mathds{1}_{s\le t} \cdot \ell_{t,k}$ and $\tilde{m}_{t,(s,k)} = \mathds{1}_{s\le t} \cdot m_{t,k}$, for all $s,k$.
    \STATE Feed $\tilde{\ell}_t$ as the loss vector and $\tilde{m}_t$ as the estimate vector to Algorithm~\ref{alg:Op_Prod} for update.
\ENDFOR
\end{algorithmic}
\end{algorithm}

The new algorithm is shown in Algorithm~\ref{alg:sleep}, which runs Algorithm~\ref{alg:Op_Prod} on $\tilde{K}=KT$ sleeping experts. Each expert is indexed by a pair $(s,k)$, with $s \in [1,T]$ and $k \in [K]$, which is asleep before step $s$ and is awake for steps $t \ge s$. Our Algorithm~\ref{alg:sleep} at step $t$ first calls Algorithm~\ref{alg:Op_Prod} to obtain the distribution $\tilde{p}_t$ over $\tilde{K}=KT$ sleeping experts, from which it computes the distribution $p_t$ over $K$ arms, with
$$p_{t,k} = \sum_{s=1}^t \tilde{p}_{t,(s,k)} / Z_t, \;\mbox{ where } Z_t = \sum_{k'=1}^K \sum_{s'=1}^t \tilde{p}_{t,(s',k')}.$$
Then it plays $p_t$ at step $t$, receives the loss vector $\ell_t$, and suffers the expected loss $\langle p_t, \ell_t \rangle$. Finally, before proceeding to the next step, it updates Algorithm~\ref{alg:Op_Prod} using the modified loss vector $\tilde{\ell}_t$, where
$$\tilde{\ell}_{t,(s,k)} =
  \begin{cases}
    \ell_{t,k}       & \quad \text{if expert $(s,k)$ is awake (when $s\le t$)},\\
    \langle p_t, \ell_t \rangle  & \quad \text{otherwise,}\\
  \end{cases}$$
as well as the estimate vector $\tilde{m}_t$, where
$$\tilde{m}_{t,(s,k)} =
  \begin{cases}
    m_{t,k}       & \quad \text{if expert $(s,k)$ is awake (when $s\le t$)},\\
    0  & \quad \text{otherwise.}\\
  \end{cases}$$
With such a definition, we can relate the instantaneous regret $\tilde{r}_{t,(s,k)}$, defined as $\langle \tilde{p}_t, \tilde{\ell}_t \rangle - \tilde{\ell}_{t,(s,k)}$, of Algorithm~\ref{alg:Op_Prod} to $r_{t,k}$ of Algorithm~\ref{alg:sleep}, defined as $\langle p_t, \ell_t \rangle - \ell_{t,k}$, as shown in the following.

\begin{lemma} \label{lem:rt}
For any $s,t$ and any $k$, $\tilde{r}_{t,(s,k)} = 0$ if $s > t$ and $\tilde{r}_{t,(s,k)} = r_{t,k}$ otherwise.
\end{lemma}

We prove the lemma in Subsection~\ref{sec:rt}. From this, together with Theorem~\ref{thm_RTk_bound}, we can bound the regret of Algorithm~\ref{alg:sleep} during an interval against some arm in the following, which we prove in Subsection~\ref{sec:R_I}.

\begin{lemma} \label{lem:R_I}
For any $t_1 \le t_2$ and any $k$, $\sum_{t=t_1}^{t_2} r_{t,k} \le \tilde{\mathcal{O}}\left(\sqrt{\sum_{t=t_1}^{t_2}\|\ell_t-\ell_{t-1}\|_\infty^2 \ln K} +\ln K \right)$.
\end{lemma}

Using this lemma, we can show the following dynamic regret bound for our Algorithm~\ref{alg:sleep}.

\begin{lemma} \label{lem:ML-dev}
For distributions with parameters $\W$ and $V$, the expected regret of Algorithm~\ref{alg:sleep}, against any offline comparator switching arms less than $N$ times, is at most
$\tilde{\mathcal{O}}\left(\sqrt{N \W \ln K} + V \sqrt{\ln K} + N \ln K\right).$
\end{lemma}


We will prove the lemma in Subsection~\ref{sec:ML-dev}. With this lemma, we are now ready to prove Theorem~\ref{theorem:dynamic_upper_full-into}. First, for switching distributions parameterized by $\Gamma$ and $\W$, we can simply choose $N=\Gamma$ in Lemma~\ref{lem:ML-dev} and have the regret bound of $\tilde{\mathcal{O}}(\sqrt{\Gamma \W \ln K} + \Gamma \ln K)$ as $V \le \Gamma$.

Next, for drifting distributions parameterized by $V$ and $\W$, we need to choose $N$ appropriately. As shown in the proof of Theorem~\ref{thm:full_info_given_V}, a regret bound with respect to comparators making less than $N$ equally-spaced switches immediately implies a regret bound with respect to fully-dynamic comparators with an additional term of $\mathcal{O}(VT/N)$. This leads to a dynamic regret bound of $\tilde{\mathcal{O}}(\sqrt{N \W \ln K} + V\sqrt{\ln K} + N\ln K) + \mathcal{O}(VT/N)$ for our algorithm. When $\W^2 \ge VT\ln K$, we can choose $N=\sqrt[3]{V^2 T^2 / (\W \ln K)}$ to give a regret bound of $\tilde{\mathcal{O}}(\sqrt[3]{\W VT \ln K})$. On the other hand, when $\W^2 < VT\ln K$, we can choose $N=\sqrt{VT / (\ln K)}$ to give a regret bound of $\tilde{\mathcal{O}}(\sqrt{VT \ln K})$. By combining these two bounds together, we have the theorem.

\subsection{Proof of Lemma~\ref{lem:rt}} \label{sec:rt}

First, we claim that $\langle \tilde{p}_t, \tilde{\ell}_t \rangle = \langle p_t, \ell_t \rangle$ for any $t$, which means that the expected loss of Algorithm~\ref{alg:sleep} equals that of Algorithm~\ref{alg:Op_Prod} at any step. To see this, note that from the definition of $\tilde{\ell}_t$, we have
$$\langle \tilde{p}_t, \tilde{\ell}_t \rangle
= \sum_{k=1}^K\sum_{s=1}^t \tilde{p}_{t,(s,k)}\ell_{t,k} + \sum_{k=1}^K \sum_{s=t+1}^T\tilde{p}_{t,(s,k)} \langle p_t, \ell_t \rangle,
$$
where the first sum above, using the definition of $p_{t,k}$,
equals
$$\sum_{k=1}^K Z_{t} p_{t,k} \ell_{t,k} = Z_{t} \langle p_t, \ell_t \rangle = \sum_{k=1}^K \sum_{s=1}^t \tilde{p}_{t,(s,k)} \langle p_t, \ell_t \rangle.$$
This implies that
$$\langle \tilde{p}_t, \tilde{\ell}_t \rangle = \sum_{k=1}^K \sum_{s=1}^T \tilde{p}_{t,(s,k)} \langle p_t, \ell_t \rangle = \langle p_t, \ell_t \rangle.$$
Then for $s>t$, we have $\tilde{r}_{t,(s,k)} = \langle \tilde{p}_t, \tilde{\ell}_t \rangle - \langle p_t, \ell_t \rangle = 0$, while for $s \le t$, we have $\tilde{r}_{t,(s,k)} = \langle \tilde{p}_t, \tilde{\ell}_t \rangle - \ell_{t,k} = r_{t,k}$, which proves the lemma.

\subsection{Proof of Lemma~\ref{lem:R_I}} \label{sec:R_I}

According to Lemma~\ref{lem:rt} and Theorem~\ref{thm_RTk_bound}, we have
$$\sum_{t=t_1}^{t_2} r_{t,k} = \sum_{t=t_1}^{t_2} \tilde{r}_{t,(t_1,k)} = \sum_{t=1}^{t_2} \tilde{r}_{t,(t_1,k)} \le \hat{\mathcal{O}}\left(\sqrt{\sum_{t=1}^{t_2} \left(\tilde{r}_{t,(t_1,k)} - \tilde{m}_{t,(t_1,k)}\right)^2 \ln \tilde{K}} +\ln \tilde{K} \right).$$
Then using Lemma~\ref{lem:rt} as well as the definition of $\tilde{m}_{t,(t_1,k)}$, we know that $\tilde{r}_{t,(t_1,k)} - \tilde{m}_{t,(t_1,k)}$ equals $0$ if $t < t_1$ and it equals $r_{t,k} - m_{t,k}$ otherwise. As $(r_{t,k} - m_{t,k})^2 \le (2 \|\ell_t - \ell_{t-1}\|_\infty)^2$ and as $\tilde{K}=KT$, we have $$\sum_{t=t_1}^{t_2} r_{t,k}= \tilde{\mathcal{O}}\left(\sqrt{\sum_{t=t_1}^{t_2} \lVert \ell_t-\ell_{t-1}\rVert_\infty^2 \ln K} +\ln K \right),$$ with the notation $\tilde{\mathcal{O}}(\cdot)$ hiding an additional $\log T$ factor. This proves the lemma.

\subsection{Proof of Lemma~\ref{lem:ML-dev}} \label{sec:ML-dev}

Consider any offline algorithm, which divides the time steps into $N$ intervals: $\I_1, \dots, \I_N$, and plays the best fixed arm in each interval. For any one such interval $\I=[t_1,t_2]$ and any arm $k$, we know from Lemma~\ref{lem:R_I} that
\begin{equation}\label{eq:expr}
\sum_{t=t_1}^{t_2} r_{t,k} \le \tilde{\mathcal{O}}\left(\sqrt{\sum_{t=t_1}^{t_2} \|\ell_t - \ell_{t-1}\|_\infty^2 \ln K} +\ln K \right).
\end{equation}
From Eq. (\ref{eq:l-mu}) in the proof of Lemma~\ref{lem:lambda}, we know that
$$\|\ell_t - \ell_{t-1}\|_\infty^2 \le 3 \left(\|\ell_t - \mu_t\|_\infty^2 + \|\mu_t - \mu_{t-1}\|_\infty^2 + \|\ell_{t-1} - \mu_{t-1}\|_\infty^2\right).$$
Using the fact that $\|\ell_\tau - \mu_\tau\|_\infty^2 \le \|\ell_\tau - \mu_\tau\|_2^2$ and $\sqrt{a+b} \le \sqrt{a} + \sqrt{b}$ for $a,b \ge 0$, we have
$$\sqrt{\sum_{t=t_1}^{t_2} \|\ell_t - \ell_{t-1}\|_\infty^2 \ln K} \le \sqrt{6 \sum_{t=t_1-1}^{t_2} \|\ell_t - \mu_t\|_2^2 \ln K} + \sum_{t=t_1}^{t_2} \|\mu_t - \mu_{t-1}\|_\infty \sqrt{3 \ln K}.$$
Then taking the expectation on both sides of (\ref{eq:expr}) and applying Jensen's inequality, we have
$$\sum_{t=t_1}^{t_2} \ex{r_{t,k}} \le \tilde{\mathcal{O}}\left(\sqrt{\sum_{t=t_1-1}^{t_2} \ex{\|\ell_t - \mu_t\|_2^2} \ln K} + \sum_{t=t_1}^{t_2} \|\mu_t - \mu_{t-1}\|_\infty \sqrt{\ln K} + \ln K \right).$$

Note that the bound above work for any arm $k$. Therefor, if we let $k_n$ be the best arm in interval $\I_n$ and sum the above over the intervals, we can bound the expected regret against such a dynamic comparator as
$$\sum_{n=1}^N \sum_{t \in \I_n} \ex{r_{t,k_n}} \le \tilde{\mathcal{O}}\left(\sqrt{N\sum_{t \in [T]} \ex{\|\ell_t - \mu_t\|_2^2} \ln K} + \sum_{t \in [T]} \|\mu_t - \mu_{t-1}\|_\infty \sqrt{\ln K} + N \ln K\right),
$$
by the Cauchy–Schwarz inequality. Then the lemma follows from the definition of $\W$ and $V$.

\section{Proof of Theorem~\ref{thm:full-low}} \label{app:full-low}

Fix any full-information algorithm $\mathcal{A}$. Let us consider two cases depending on whether or not $\W > \Gamma$. First, for the case with $\W > \Gamma$, we divide the time horizon into $\Gamma$ intervals of equal length. In the first $\W/\Gamma$ time steps of each interval, according to the lower bound proof in \cite{auer2002nonstochastic}, there exists a stationary loss distribution of variance $1$ for making the regret of any algorithm $\mathcal{A}$ at least $\Omega(\sqrt{\W/ \Gamma})$. For the rest of the intervals we use a fixed distribution with no variance. With this choice of distributions, the total variance does not exceed $\Lambda$, and the regret lower bound is $\Gamma \cdot \Omega(\sqrt{\W/ \Gamma}) \ge \Omega(\sqrt{\Gamma\W})\geq \Omega(\sqrt{\Gamma\W}+\Gamma)$.

Next, for the case with $\Gamma \ge \W$, we show a lower bound of $\Omega(\Gamma) \ge \Omega(\sqrt{\Gamma\W}+\Gamma)$. Let $\ell^{(i)}$, for $i \in [K]$, be the loss vector with $0$ in dimension $i$ and $1$ elsewhere, which can serve as a distribution with mean $\ell^{(i)}$ and variance $0$, with arm $i$ being the optimal one.
Then one can show the existence of a sequence of $\Gamma$ loss vectors, chosen from these $K$ vectors, which makes the expected regret of $\mathcal{A}$ at least $\Gamma(K-1)/K$ for the first $\Gamma$ steps, by a probabilistic argument. Formally, by choosing one of the loss vectors uniformly and independently for each step, the expected loss of $\mathcal{A}$ is $\Gamma(K-1)/K$, which implies the existence of a fixed sequence of $\Gamma$ loss vectors achieving this bound. On the other hand, the total loss of a fully dynamic offline algorithm is clearly $0$, which implies that the expected regret of $\mathcal{A}$ with respect to this fixed sequence of $\Gamma$ loss vectors is at least $\Gamma(K-1)/K = \Omega(\Gamma)$. The last loss vector is then kept for the remaining $T-\Gamma$ steps, which cannot decrease the regret, and the total number of switches is at most $\Gamma-1$. This completes the proof of the theorem.

\section{Proof of Lemma~\ref{lemma:sigma1}} \label{app:sigma1}

We let $\mathcal{Q}$ and $\mathcal{P}$ both be Bernoulli-type distributions taking values from $\{0, 2\sigma\}$. For $\mathcal{Q}$, we let it take each value with equal probability so that it has mean $\sigma$ and variance $\sigma^2$. For $\mathcal{P}$, we make it have mean $\sigma - \epsilon$, by letting it take the value $0$ with probability $\frac{\sigma + \epsilon}{2\sigma}$ and the value $2\sigma$ with probability $\frac{\sigma - \epsilon}{2\sigma}$. Then the variance of $\mathcal{P}$ is $\frac{\sigma + \epsilon}{2\sigma} (\sigma - \epsilon)^2 + \frac{\sigma - \epsilon}{2\sigma} (\sigma + \epsilon)^2 = (\sigma-\epsilon)(\sigma+\epsilon)<\sigma^2$. Thus, the first two conditions are satisfied. For the last condition, note that
$$(\ln 2) \mathrm{KL}(\mathcal{Q}, \mathcal{P}) = \frac{1}{2}\ln\left(\frac{\frac{1}{2}}{\frac{\sigma-\epsilon}{2\sigma}}\right) + \frac{1}{2}\ln\left(\frac{\frac{1}{2}}{\frac{\sigma+\epsilon}{2\sigma}}\right)
= \frac{1}{2}\ln\left(\frac{\sigma^2}{\sigma^2-\epsilon^2}\right) = \frac{1}{2}\ln\left(1+\frac{\epsilon^2}{\sigma^2-\epsilon^2}\right)
$$
which, using the fact that $1+x \le e^x$ and $\epsilon^2 \le \sigma^2/2$, is at most $\frac{\epsilon^2}{2(\sigma^2-\epsilon^2)} \le \frac{\epsilon^2}{\sigma^2}$ as required.

\end{document}